\documentclass[runningheads]{llncs}

\usepackage{graphicx}
\usepackage{amsmath,amssymb} 
\usepackage{color}
\usepackage{algorithm}
\usepackage[noend]{algpseudocode}
\usepackage{bbding}
\usepackage{pifont}
\usepackage{hyperref}
\usepackage{capt-of}
\newtheorem{theorem*}{Theorem}

\newif\ifarxiv
\arxivtrue

\newcommand{\cmark}{\ding{51}}
\newcommand{\xmark}{\ding{55}}
\makeatletter
\newcommand{\printfnsymbol}[1]{%
  \textsuperscript{\@fnsymbol{#1}}%
}
\makeatother
\begin{document}

\title{Deep Expander Networks:\\Efficient Deep Networks from Graph Theory} 

\titlerunning{Deep Expander Networks}
\author{Ameya Prabhu\thanks{ indicates these authors contributed equally to this work.} \qquad
Girish Varma\printfnsymbol{1} \qquad
Anoop Namboodiri}

\authorrunning{A. Prabhu, G. Varma and A. Namboodiri}

\institute{Center for Visual Information Technology\\
Kohli Center on Intelligent Systems, 
IIIT Hyderabad, India \\
\email{ameya.pandurang.prabhu@gmail.com, \{girish.varma, anoop\}@iiit.ac.in}\\
\url{https://github.com/DrImpossible/Deep-Expander-Networks}}

\maketitle

\begin{abstract}
Efficient CNN designs like ResNets and DenseNet were proposed to improve accuracy vs efficiency trade-offs. They essentially increased the connectivity, allowing efficient information flow across layers. Inspired by these techniques, we propose to model connections between filters of a CNN using graphs which are simultaneously sparse and well connected. Sparsity results in efficiency while well connectedness can preserve the expressive power of the CNNs. We use a well-studied class of graphs from theoretical computer science that satisfies these properties known as Expander graphs. Expander graphs are used to model connections between filters in CNNs to design networks called X-Nets. We present two guarantees on the connectivity of X-Nets: Each node influences every node in a layer in logarithmic steps, and the number of paths between two sets of nodes is proportional to the product of their sizes. We also propose efficient training and inference algorithms, making it possible to train deeper and wider X-Nets effectively.\\

Expander based models give a $4\%$ improvement in accuracy on MobileNet over grouped convolutions, a popular technique, which has the same sparsity but worse connectivity. X-Nets give better performance trade-offs than the original ResNet and DenseNet-BC architectures. We achieve model sizes comparable to state-of-the-art pruning techniques using our simple architecture design, without any pruning. We hope that this work motivates other approaches to utilize results from graph theory to develop efficient network architectures.
\end{abstract}

\section{Introduction}
\begin{figure}[t]
\centering
\includegraphics[scale=0.175]{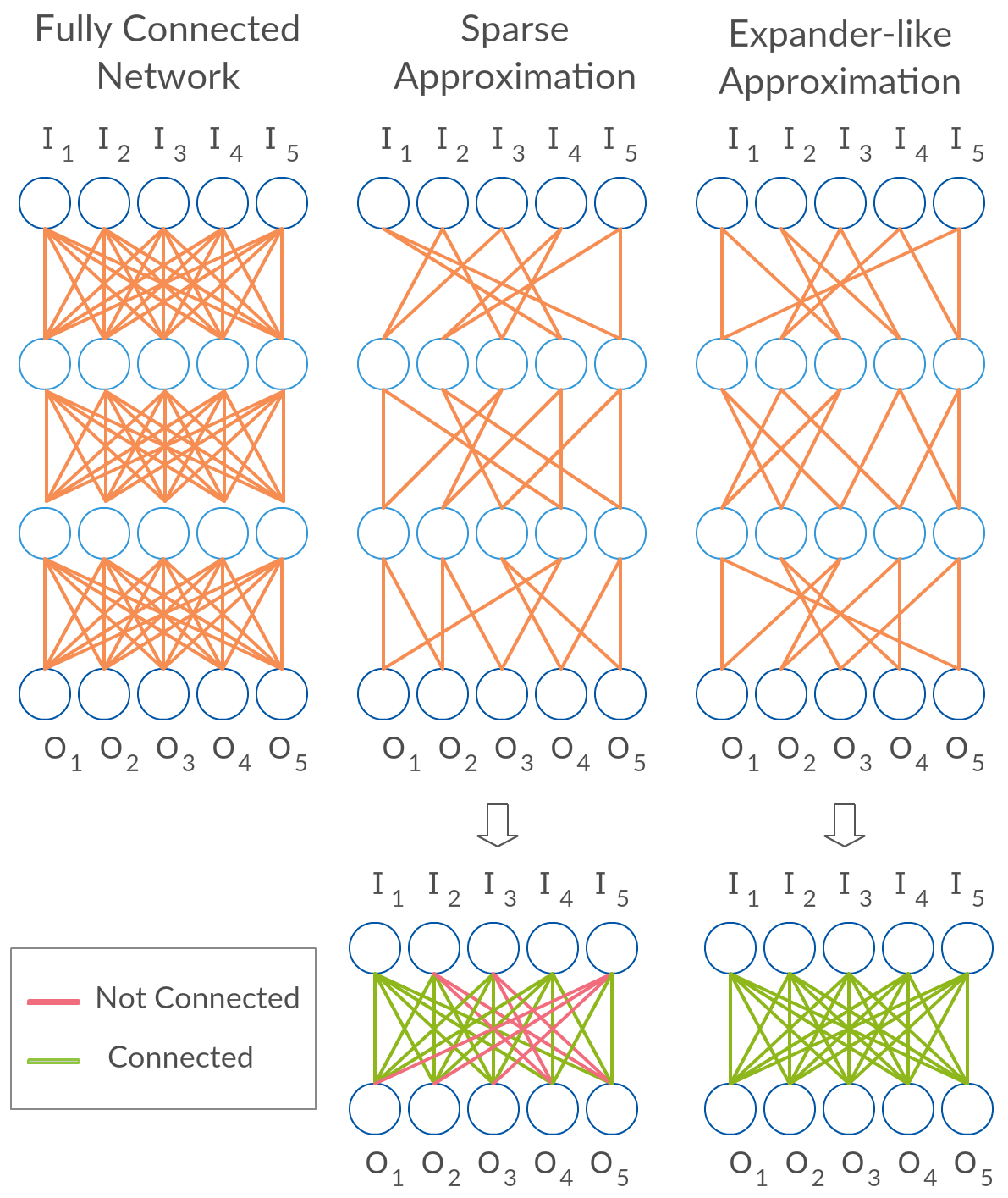}

\caption{Popular sparse approximations are agnostic to the global information flow in a network, possibly creating disconnected components. In contrast, expander graph-based models produce sparse yet highly connected networks.}
\label{fig:intro}

\end{figure}

Convolutional Neural Networks (CNNs) achieve state-of-the-art results in a variety of machine learning applications\cite{krizhevsky2012imagenet,he2016identity,huang2016densely,googlenet}. However, they are also computationally intensive and consume a large amount of computing power and runtime memory. After the success of VGG Networks \cite{vgg}, there has been significant interest in designing compact neural network architectures due to the wide range of applications valuing mobile and embedded devices based use cases.
 
ResNet\cite{resnet} and DenseNet-BC\cite{huang2016densely} directed the focus of efficient designs of convolutional layers on increasing connectivity. Additional connectivity with residual connections to previous layers provided efficient information flow through the network, enabling them to achieve an order of magnitude reduction in storage and computational requirements. We take inspiration from these approaches, to focus on designing highly connected networks. We explore making networks efficient by designing sparse networks that preserve connectivity properties. 
Recent architectures like MobileNet\cite{howard2017mobilenets} improves the efficiency by an order of magnitude over a ResNet. However, in order to achieve this, they sparsify a network by removing several connections from a trained network, reducing their accuracies in the process. We ask a basic question: If we try to maximize the connectivity properties and information flow, can we achieve the same efficiency gains with minimal loss in accuracy? 
It is essential that the connections allow information to flow through the network easily. That is, each output node must at least have the capacity to be sensitive to features of previous layers. As we can see from Fig.\ref{fig:intro}, traditional model compression techniques such as pruning can aggravate the problem, since they can prune the neuron connections of a layer, while being agnostic of global connectivity of the network. A necessary condition for having good representational power is efficient information flow through the network, which is particularly suited to be modeled by graphs. 
We propose to make the connections between neurons (filters in the case of CNNs) according to specific graph constructions known as expander graphs. They have been widely studied in spectral graph theory \cite{sgt} and pseudorandomness \cite{pseudo}, and are known to be sparse but highly connected graphs. Expander graphs have a long history in theoretical computer science, also being used in practice in computer networks, constructing error correcting codes, and in cryptography (for a survey, see \cite{expsurvey}).

{\bf Main Contributions:} i.) We propose to represent neuronal connections in deep networks using expander graphs (see Section \ref{sec:approach}). We further prove that X-Nets have strong connectivity properties (see Theorem \ref{thm:conn}). ii.) We provide memory-efficient implementations of  Convolutional (X-Conv) layers using sparse matrices and propose a fast expander-specific algorithm (see Section \ref{sec:implementation}). iii.) We empirically compare X-Conv layers with grouped convolutions that have the same level of sparsity but worse connectivity. X-Conv layers obtain a 4\% improvement in accuracy when both the techniques are applied to the MobileNet architecture trained on Imagenet (see Section \ref{sec:group}).  iv.) We also demonstrate the robustness of our approach by applying the technique to some of the state of the art models like DenseNet-BC and ResNet, obtaining better performance trade-offs (see Section \ref{sec:denres}). v.) Additionally, our simple design achieves comparable compression rates to even the state-of-the-art trained pruning techniques. (see Section \ref{sec:prun}). vi.) Since we enforce the sparsity before the training phase itself, our models are inherently compact and faster to train compared to pruning techniques. We leverage this and showcase the performance of wider and deeper X-Nets (see Section \ref{sec:ultrawide}). 

\section{Related Work}\label{sec:rel-works}
Our approach lies at the intersection of trained pruning techniques and efficient layer design techniques. We present a literature survey regarding both the directions in detail. 

\subsection{Efficient Layer Designs}

Currently there is extensive interest in developing novel convolutional layers/blocks and effectively leveraging them to improve architectures like \cite{squeezenet,howard2017mobilenets,hu2017squeeze}. Such micro-architecture design is in a similar direction as our work. In contrast, approaches like \cite{googlenet} try to design the macro-architectures by connecting pre-existing blocks. Recent concurrent work has been on performing architecture searches effectively \cite{liu2017hierarchical,zoph2017learning,zhong2017practical,liu2017progressive}. Our work is complementary to architecture search techniques as we can leverage their optimized macro-architectures. 

Another line of efficient architecture design is Grouped Convolutions: which was first proposed in AlexNet\cite{krizhevsky2012imagenet}, recently popularized by MobileNets\cite{howard2017mobilenets} and XCeption\cite{xception} architectures . This is currently a very active area of current research, with a lot of new concurrent work being proposed \cite{zhang2017shufflenet,sandler2018inverted,huang2017condensenet}. 

It is interesting to note that recent breakthroughs in designing accurate deep networks \cite{resnet,huang2016densely,xie2017aggregated} were mainly by introducing additional connectivity to enable the efficient flow of information through deep networks. This enables the training of compact, accurate deep networks. These approaches, along with Grouped Convolutions are closely related to our approach. 
% \vspace{-0.2cm}
\subsection{Network Compression} 
Several methods have been introduced to compress pre-trained networks as well as train-time compression. Models typically range from low-rank decomposition  \cite{sainath2013low,novikov2015tensorizing,masana2017domain} to network pruning  \cite{blundell2015weight,liu2015sparse,he2017channel,molchanov2016pruning}.

There is also a major body of work that quantizes the networks at train-time to achieve efficiency \cite{rastegari2016xnor,courbariaux2016bnn,han2015deep,wu2016quantized,bagherinezhad2016lcnn,zhu2017trained,zhou2016dorefa}. The problem of pruning weights in train-time have been extensively explored  \cite{wen2016structured,li2016pruning} primarily from weight-level \cite{lebedev2016fast,scardapane2017group,srinivas2015data,guo2016dynamic} to channel-level pruning \cite{liu2017learning,li2016pruning,wen2016structured}. Weight-level pruning has the highest compression rate while channel-level pruning is easier to practically exploit and has compression rates almost on par with the former. Hence, channel-level pruning is currently considered superior \cite{liu2017learning}. Channel-level pruning approaches started out with no guidance for sparsity \cite{chen2015compressing} and eventually added constraints \cite{srinivas2017training,zhou2016less,yoon2017combined}, tending towards more structured pruning.

However, to the best of our knowledge, this is the first attempt at constraining neural network connections by graph-theoretic approaches to improve deep network architecture designs. Note that we do not prune weights during training.  

\section{Approach}\label{sec:approach}
Recent breakthroughs in CNN architectures like ResNet\cite{he2015convolutional} and DenseNet-BC\cite{huang2016densely} are ideas  based on increasing connectivity, which resulted in better performance trade-offs. These works suggest that connectivity is an important property for improving the performance of deep CNNs. In that vein, we investigate ways of preserving connectivity between neurons while significantly sparsifying the connections between them. Such networks are expected to preserve accuracy (due to connectivity) while being runtime efficient (due to the sparsity). We empirically demonstrate this in the later sections.

\subsection{Graphs and Deep CNNs}

We model the connections between neurons as graphs. This enables us to leverage well-studied concepts from Graph Theory like Expander Graphs. Now, we proceed to formally describe the connection between graphs and Deep CNNs. 

{\bf Linear Layer defined by a Graph:} Given a bipartite graph $G$ with vertices $U, V$, the Linear layer defined by $G$, is a layer with $|U|$ input neurons, $|V|$ output neurons and each output neuron $v \in V$ is only connected to the neighbors given by $G$. Let the graph $G$ be sparse, having only $M$ edges. Then this layer has only $M$ parameters as compared to $|V|\times |U|$, which is the size of typical linear layers. 
{\bf Convolutional Layer defined by a Graph:} Let a Convolutional layer be defined as a bipartite graph $G$ with vertices $U,V$ and a window size of $c\times c$. This layer takes a 3D input with $|U|$ channels and produces a 3D output with $|V|$ channels. The output channel corresponding to a vertex $v \in V$ is computed only using the input channels corresponding the the neighbors of $v$. Let $G$ be sparse, having only $M$ edges. Hence the kernel of this convolutional layer has $M \times c \times c$ parameters as compared to $|V|\times |U| \times c \times c$, which is the number of parameters in a vanilla CNN layer.

\subsection{Sparse Random Graphs}

We want to constrain  our convolutional layers to form a sparse graph $G$. Without any prior knowledge of the data distribution, we take inspiration from randomized algorithms and propose choosing the neighbours of every output neuron/channel uniformly and independently at random from the set of all its input channels. It is known that a graph $G$ obtained in this way belongs to a well-studied category of graphs called Expander Graphs, known to be sparse but well connected. 

{\bf Expander Graph:} A bipartite expander with degree $D$ and spectral gap $\gamma$, is a bipartite graph $G=(U,V,E)$ ($E$ is the set of edges,  $ E \subseteq U\times V$) in which: 

{\bf1.) Sparsity:} Every vertex in $V$ has only $D$ neighbors in $U$. We  will be using constructions with $D << |U|$. Hence the number of edges is only $D \times |V|$ as compared to $|U| \times |V|$ in a dense graph.

 {\bf 2.) Spectral Gap:} The eigenvalue with  the second largest absolute value $\lambda$ of the adjacency matrix is bounded away from D (the largest eigenvalue). Formally $1-\lambda/D \geq \gamma.$

{\bf Random expanders:} A random bipartite expander of degree $D$ on the two vertex sets $U, V$, is a graph in which for every vertex $v \in V$, the $D$ neighbors are chosen independently and uniformly from $U$. It is a well-known result in graph theory that such graphs have a large spectral gap (\cite{pseudo}).
Similar to random expanders, there exist several explicit expander constructions. More details about explicit expanders can be found in the supplementary section.
We now proceed to give constructions of deep networks that have connections defined by an expander graph.

{\bf Expander Linear Layer (X-Linear):} The Expander Linear (X-Linear) layer is a layer defined by a random bipartite expander $G$ with degree $D$. The expander graphs that we use have values of $D << |U|$, while having an expansion factor of $K \approx D$, which ensures that the layer still has good expressive power.

{\bf Expander Convolutional Layer (X-Conv):} The Expander Convolutional (X-Conv) layer is a convolutional layer defined by a random bipartite expander graph $G$ with degree $D$, where $D << |U|$. 

{\bf Deep Expander Networks (X-Nets):} Given expander graphs 
$$G_1 = (V_0,V_1,E_1), G_2 = (V_1,V_2,E_2), \cdots, G_t = (V_{t-1},V_t,E_t)$$, we define the Deep Expander Convolutional Network (Convolutional X-Net or simply X-Net) as a $t$ layer deep network in which the convolutional layers are replaced by X-Conv layers and linear layers are replaced by X-Linear layers defined by the corresponding graphs. 

\begin{figure*}[t]
\centering
\includegraphics[scale=0.15]{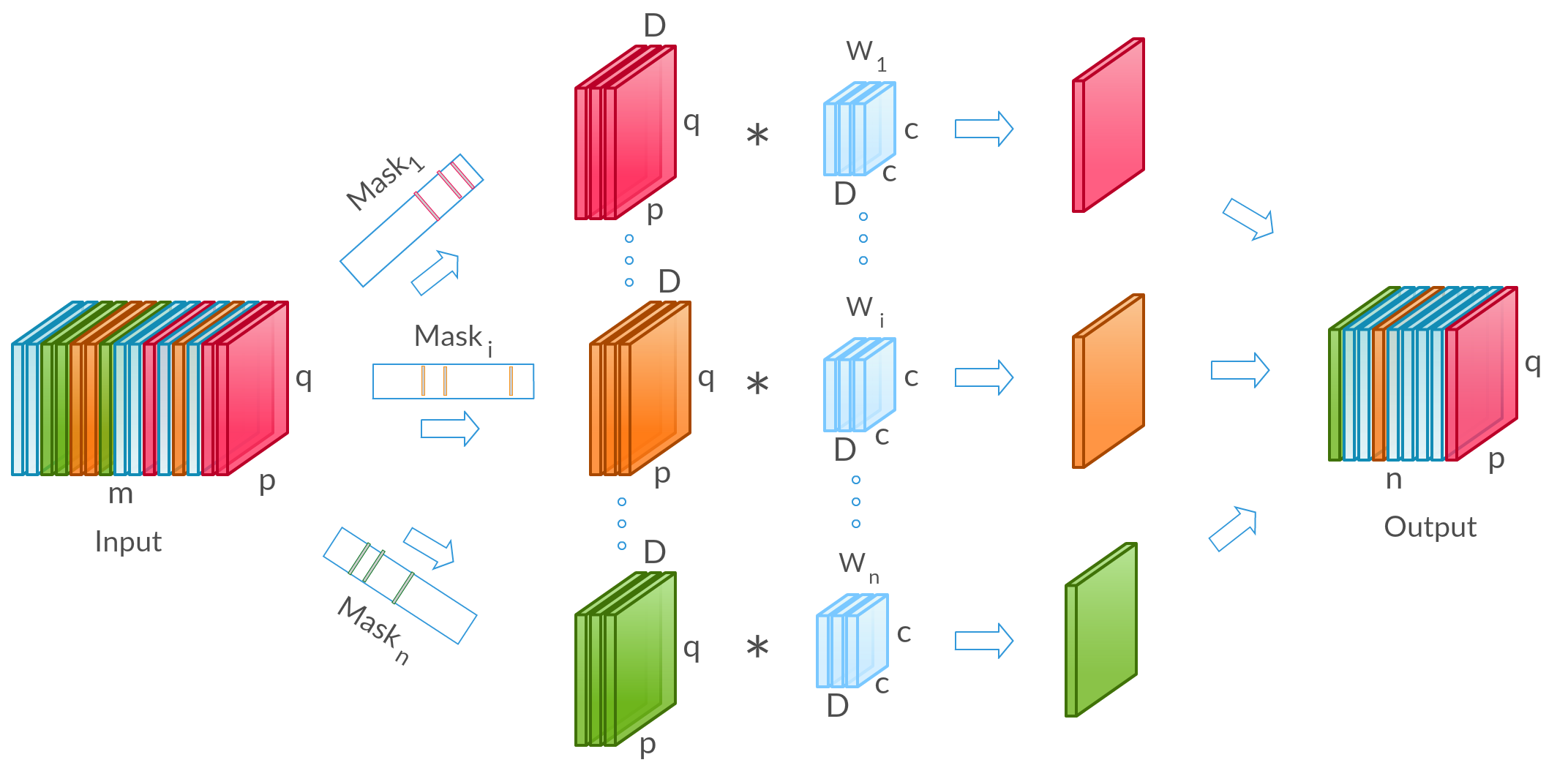}

\caption{The proposed fast convolution algorithm for X-Conv layer. We represent all the non-zero filters in the weight matrix of the X-Conv layer as a compressed dense matrix of $D$ channels. The algorithm starts by selecting $D$ channels from input (with replacement) using a mask created while initializing the model. The output is computed by convolving these selected channels with the compressed weight matrices.}
\label{fig:efficientmatrix}

\end{figure*}

\subsection{Measures of Connectivity}

In this subsection, we describe some connectivity properties of Expander graphs (see \cite{pseudo}, for the proofs). These will be used to prove the properties of sensitivity and mixing of random walks in X-Nets .

{\bf Expansion:} For every subset $S \subseteq V$ of size $\leq \alpha |V|$ ($\alpha \in (0,1)$ depends on the construction), let $N(S)$ be the set of neighbors. Then $|N(S)| \geq K |S|$ for $K\approx D$. That is, the neighbors of the vertices in $S$ are almost distinct. It is known that random expanders have expansion factor $K \approx D$ (see Theorem 4.4 in \cite{pseudo}).

{\bf Small Diameter:} The diameter of a graph is the length of the longest path among all shortest paths. If $G(U,V,E)$ is a $D$-regular expander with expansion factor $K > 1$ and diameter $d$, then  $d\leq O(\log n )$. This bound on the diameter implies that for any pair of vertices, there is a path of length $O(\log n)$ in the graph.

{\bf Mixing of Random Walks:} Random walks in the graph quickly converge to the uniform distribution over nodes of the graph. If we start from any vertex and keep moving to a random neighbor, in $O(\log n)$ steps the distribution will be close to uniform over the set of vertices.

\subsection{Sensitivity of X-Nets}

X-Nets have multiple layers, each of which have connections derived from an expander graph. We can guarantee that the output nodes in such a network are sensitive to all the input nodes.

\begin{theorem}[Sensitivity of X-Nets]\label{thm:conn}
Let $n$ be the number of input as well as output nodes in the network and $G_1,G_2,\cdots, G_t$ be $D$ regular bipartite expander graphs with $n$ nodes on both sides. Then  
every output neuron is sensitive to every input in a Deep X-Net defined by $G_i$'s with depth $t = O( \log n)$.
\end{theorem}
\begin{proof}
For every pair of input and output $(u,v)$, we show that there is a path in the X-Net. The proof is essentially related to the the fact that expander graphs have diameter $O(\log n)$. A detailed proof can be found in the supplementary material.
\end{proof}

Next, we show a much stronger connectivity property known as mixing for the X-Nets. The theorem essentially says that the number of edges between subsets of input and output nodes is proportional to the product of their sizes. This result implies that the connectivity properties are uniform and rich across all nodes as well as subsets of nodes of the same size. Simply put, all nodes tend to have equally rich representational power. 

\begin{theorem}[Mixing in X-Nets]
Let $n$ be the number of input as well as output nodes in the network and $G$ be $D$ regular bipartite expander graph with $n$ nodes on both sides. Let $S,T$ be subsets of input and output nodes in the X-Net layer defined by $G$. The number of edges between $S$ and $T$ is $\approx D|S||T|/n$
\end{theorem}
\begin{proof}
A detailed proof is provided in the supplementary material.
\end{proof}

\section{Efficient Algorithms}\label{sec:implementation}
 
In this section, we present efficient algorithms of X-Conv layers. Our algorithms achieve speedups and save memory in the training as well as the inference phase. This enables one to experiment with significantly wider and deeper networks given  memory and runtime constraints.
We exploit the structured sparsity of expander graphs to design fast algorithms.  We propose two methods of training X-Nets, both requiring substantially less memory and computational cost than their vanilla counterparts: \\1) Using Sparse Representations \\2) Expander-Specific Fast Algorithms.

\subsection{Using Sparse Representation}

The adjacency matrices of expander graphs are highly sparse for $D << n$. Hence, we can initialize a sparse matrix with non-zero entries corresponding to the edges of the expander graphs. Unlike most pruning techniques, the sparse connections are determined before training phase, and stay fixed. Dense-Sparse convolutions are easy to implement, and are supported by most deep learning libraries. CNN libraries like Cuda-convnet\cite{cudaconvnet} support such random sparse convolution algorithms.

 \begin{algorithm}[t]
 \textbf{Algorithm 1: } Fast Algorithm for Convolutions in X-Conv Layer\\
 \begin{algorithmic}[1]
 \State For every vertex $v \in \{1,\cdots, n\}$, let $N(v,i)$ denote the $i$th neighbor of $v$ ($i \in \{1,\cdots, D\}$).
 \State Let $K_v$ be the $c\times c \times D \times 1$ sized kernel associated with the $v$th output channel.
 \State Let $O_v[x,y]$ be the output value of the $v$th channel at the position $x,y$.
 \For{$v$= 1 to n}
     \State $O_v[x,y] = K_v * \textit{Mask}_{N(v,1),\cdots N(v,D)}(I)[x,y]$.
 \EndFor
 \end{algorithmic}
 \label{alg:cnnalgo}
 \end{algorithm}

\subsection{X-Net based Fast Dense Convolution}
Next, we present fast algorithms that exploit the sparsity of expander graphs.\\ %Moreover, our algorithm for X-Conv layers works for random expanders as well.

{\bf X-Conv:} In an X-Conv layer, every output channel is only sensitive to $out$ rom input channels. We propose to use a mask to select $D$ channels of the input, and then convolve with a $c \times c \times D \times 1$ kernel, obtaining a single channel per filter in the output. The mask is obtained by choosing D samples uniformly (without replacement) from the set $\{1,\cdots N\}$, where $N$ is the number of input channels. The mask value is $1$ for each of the selected $D$ channels and $0$ for others (see Algorithm \ref{alg:cnnalgo}). This is illustrated in Figure \ref{fig:efficientmatrix}. There has been recent work about fast CUDA implementations called Block-Sparse GPU Kernels\cite{blocksparse}, which can implement this algorithm efficiently.

\section{Experiments and Results}\label{sec:experiments}
\begin{figure}
\centering
\includegraphics[scale=0.3]{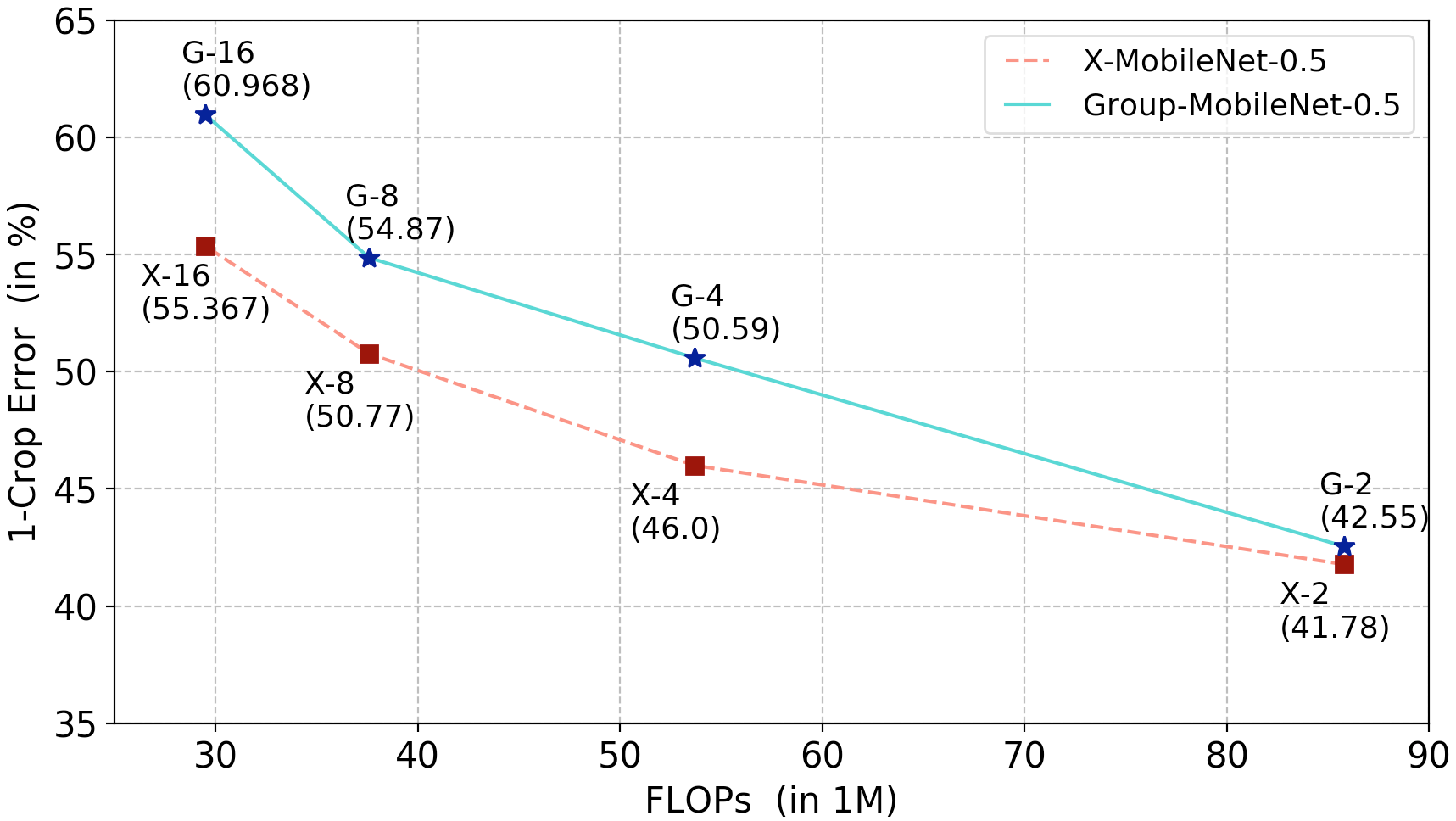}
\caption{Comparison between Grouped convolutions and X-Conv using MobileNet architecture trained on ImageNet. X-$d$ or G-$d$ represents the 1x1 conv layers are compressed by $d$ times using X-Conv or Groups. We observe X-MobileNets beat Group-MobileNet by 4\% in accuracy on increasing sparsity.}
\label{fig:mobilenet}
\end{figure}
In this section, we benchmark and empirically demonstrate the effectiveness of X-Nets on a variety of CNN architectures. Our code is available at: \url{https://github.com/DrImpossible/Deep-Expander-Networks}. 
\subsection{Comparison with Grouped Convolution}
\label{sec:group}

First, we compare our Expander Convolutions (X-Conv) against Grouped Convolutions (G-Conv). We choose G-Conv as it is a popular approach, on which a lot of concurrent works \cite{zhang2017shufflenet} have developed their ideas. G-Conv networks have the same sparsity as X-Conv networks but lack only the connectivity property. This will test whether increasing connectivity increases accuracy, i.e does a graph without good connectivity properties provides worse accuracy? We choose MobileNet as the base model for this experiment, since it is the state-of-the-art in efficient CNN architectures.
\begin{figure*}[t] 
\begin{tabular}{cc}
 \includegraphics[scale=0.28]{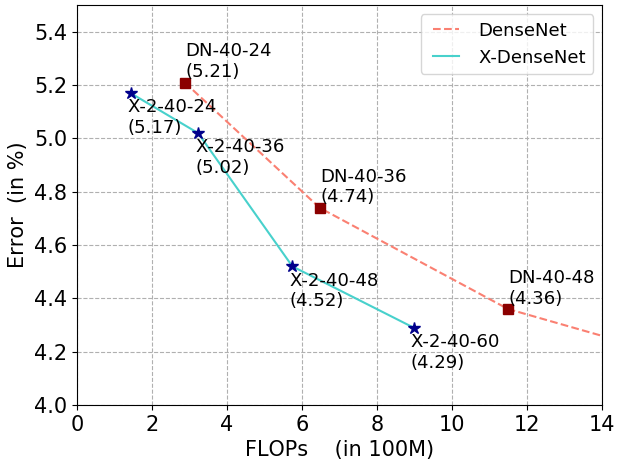}  & \includegraphics[scale=0.28] {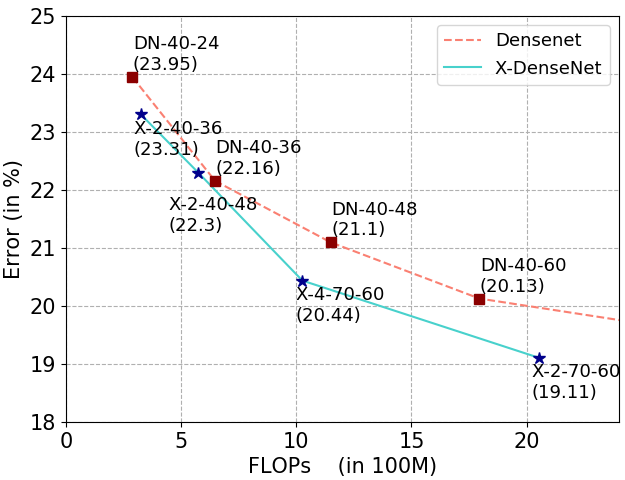} \\
\\
(a) CIFAR10 & (b) CIFAR100 \\
\end{tabular}

\caption{We show the error as a function of \#FLOPs during test-time (below) for DenseNet-BC with X-DenseNet-BCs on CIFAR10 and CIFAR100 datasets. We observe X-DenseNet-BCs achieve better performance tradeoffs over DenseNet-BC models. For each datapoint, we mention the X-C-D-G notation (see Section \ref{sec:denres}) along with the accuracy.
}

\label{fig:cifar}
\end{figure*}
We compare X-Conv against grouped convolutions using MobileNet-0.5 on the ImageNet classification task. We replace the $1\times 1$ convolutional layers in MobileNet-0.5 with X-Conv layers forming X-MobileNet-0.5. Similarly, we replace them with G-Conv layers to form Group-MobileNet-0.5. Note that we perform this only in layers with most number of parameters (after the 8th layer as given in Table 1 of \cite{howard2017mobilenets}). We present our results in Figure \ref{fig:mobilenet}. The reference original MobileNet-0.5 has an error of 36.6\% with a cost of 150M FLOPs. Additional implementation details are given in the supplementary material.

We can observe that X-MobileNets beat Group-MobileNets by over 4\% in terms of accuracy when we increase sparsity. This also demonstrates that X-Conv can be used to further improve the efficiency of even the most efficient architectures like MobileNet.

\subsection{Comparison with Efficient CNN Architectures}
\label{sec:denres}

%We test whether our approach can improve the performance tradeoff in popular architectures like ResNet and DenseNet 
In this section, we test whether Expander Graphs can improve the performance trade-offs even in state-of-the-art architectures such as DenseNet-BCs \cite{huang2016densely} and ResNets\cite{he2015convolutional} on the ImageNet \cite{deng2009imagenet} dataset. We additionally train DenseNet-BCs on CIFAR-10 and CIFAR-100 \cite{krizhevsky2009learning} datasets to demonstrate the robustness of our approach across datasets. 

Our X-ResNet-C-D is a $D$ layered ResNet that has every layer except the first and last replaced by an X-Conv layer that compresses connections between it and the previous layer by a factor of $C$. We compare across various models like ResNets-34,50,101. Similarly, our X-DenseNet-BC-C-D-G architecture has depth $D$, and growth rate $G$. We use DenseNet-BC-121-32,169-32,161-48,201-32 as base models. These networks have every layer except the first and last replaced by an X-Conv layer that compresses connections between it and the previous layer by a factor of $C$. More details are provided in the supplementary material.
\begin{minipage}{\textwidth}
%  \vspace{2em}
  \begin{minipage}[b]{0.45\textwidth}
    \centering
    \includegraphics[scale=0.25]{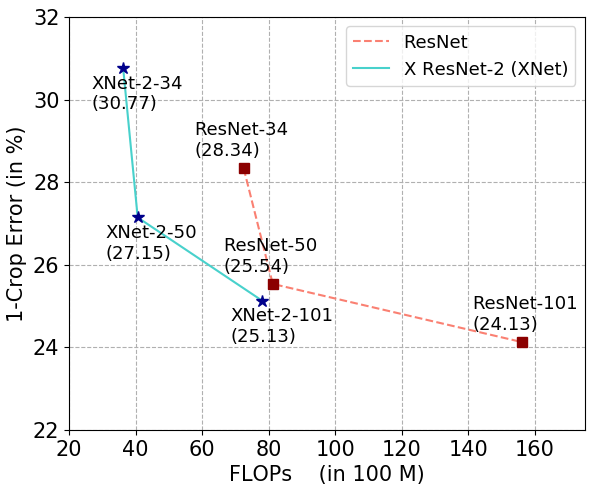}
\captionof{figure}{We show the error as a function of \#FLOPs to compare between ResNet and X-ResNet on the ImageNet dataset. We observe X-ResNets achieve better performance tradeoffs over original ResNet models.}
\label{fig:resnet}
 
  \end{minipage}
  \hfill
  \begin{minipage}[b]{0.5\textwidth}
    
    \resizebox{0.9\columnwidth}{!}{
    \begin{tabular}{|l|c|c|}
\hline
{\bf Model} & {\bf Accuracy}  & {\bf \#FLOPs}\\
\hline
ResNet &   & {\bf(in 100M)}\\
\hline
X-ResNet-2-34 & 69.23\%  & 35\\
X-ResNet-2-50 & {\bf 72.85\%}  & {\bf 40}\\
ResNet-34 & 71.66\%  & 70\\
X-ResNet-2-101 & {\bf 74.87\%}  & {\bf 80}\\
ResNet-50 & 74.46\%  & 80\\
ResNet-101 & 75.87\%  & 160\\
\hline
DenseNet-BC &  &   \\
\hline
X-DenseNet-BC-2-121 & 70.5\% & 28\\
X-DenseNet-BC-2-169 & 71.7\% & 33\\
X-DenseNet-BC-2-201 & 72.5\% & 43\\
X-DenseNet-BC-2-161 & {\bf 74.3\%} & {\bf 55}\\
DenseNet-BC-121 & 73.3\% & 55\\
DenseNet-BC-169 & 74.8\% & 65\\
DenseNet-BC-201 & 75.6\% & 85\\
DenseNet-BC-161 & 76.3\% & 110\\
\hline
\end{tabular}}
      \captionof{table}{Results obtained by ResNet and DenseNet-BC models on ImageNet dataset, ordered by \#FLOPs. or each datapoint, we use the X-C-D-G notation (see Section \ref{sec:denres}) along with the accuracy.}
\label{tab:imagenet}
    \end{minipage}
    
  \end{minipage}
\subsection{Comparison with Pruning Techniques}
\label{sec:prun}
We plot the performance tradeoff of X-ResNets against ResNets in Figure \ref{fig:resnet} . We achieve significantly better performance tradeoffs compared to the original model. More specifically, we can reduce the \#FLOPs in ResNets by half while incurring only 1-1.5\% decrease in accuracy. Also, we can compare models with similar \#FLOPs or accuracy with the help of Table \ref{tab:imagenet}. We observe that  X-ResNet-2-50 has 43\% fewer FLOPs than ResNet-34, but achieves a 1\% improvement in accuracy against it. Similarly, X-DenseNet-BC-2-161 has similar \#FLOPs as DenseNet-BC-121, but achieves a 1\% improvement in accuracy.  

To further prove the robustness of our approach on DenseNet-BC, we test the same on CIFAR10 and CIFAR100, and plot the tradeoff curve in Figure \ref{fig:cifar}.
We observe that we can achieve upto 33\% compression keeping accuracy constant on CIFAR-10 and CIFAR-100 datasets. \\

We compare our approach with methods which prune the weights during or after training. Our method can be thought of as constraining the weight matrices with a well studied sparse connectivity pattern even before the training starts. This results in fast training for the compact X-Conv models, while the trained pruning techniques face the following challenges: \\

\noindent 1) Slow initial training due to full dense model.\\
2) Several additional phases of pruning and retraining.\\

Hence they achieve the compactness and runtime efficiency only in test time. Nevertheless we show similar sparsity can be achieved by our approach without explicitly pruning. We benchmark on VGG16 and AlexNet architectures since most previous results in the pruning literature have been reported on these architectures. In Table \ref{tab:cifar_fullcomp}, we compare two X-VGG-16 models against existing pruning techniques. We achieve comparable accuracies to the previous state-of-the-art model with 50\% fewer parameters and \#FLOPs. Similarly, in Table \ref{tab:imagenet_fullcomp} we compare X-AlexNet with trained pruning techniques on the Imagenet dataset. Despite having poor connectivity due to parameters being concentrated only in the last three fully connected layers, 
we achieve similar accuracy to AlexNet model using only 7.6M-9.7M parameters out of 61M, comparable to the state-of-the-art pruning techniques which have upto 3.4M-5.9M parameters. Additionally, it is possible to improve compression by applying pruning methods on our compact architectures, but pruning X-Nets is out of the scope of our current work.
\begin{table}[t]%{0.8\textwidth}
  \centering
\begin{tabular}{|l|c|c|c|}
\hline
{\bf Method} & {\sc { \bf Accuracy}} & {\bf \#Params} & {\bf Training }  \\
 %&  &  & {\bf Speedup }?\\
\hline
Li et al. \cite{li2016pruning} & 93.4\%  & 5.4M &  \xmark \\
Liu et al. \cite{liu2017learning} &  93.8\% & 2.3M &  \xmark \\
\hline
X-VGG16-1 & 93.4\%  & 1.65M (9x) &  \cmark \\
X-VGG16-2 & 93.0\%  & 1.15M (13x) &  \cmark \\
\hline
VGG16-Orig &  94.0\% & 15.0M & - \\
\hline
\end{tabular}
\captionof{table}{Comparison with other methods on CIFAR-10 dataset using VGG16 as the base model. We significantly outperform popular compression techniques, achieving similar accuracies with upto 13x compression rate.}
\label{tab:cifar_fullcomp}
%\vspace{-0.6cm}
  \end{table}
  
\begin{table}[t]%{0.8\textwidth}
    \centering
\begin{tabular}{|l|c|c|c|}
\hline
{\bf Method} & {\bf Accuracy} & {\bf \#Params} & {\bf Training} \\
 & {\bf } &  & {\bf  Speedup}?\\
%\hline
%AlexNet-FullPrec & 57.2\% & 61M & - \\
\hline
Network Pruning & & & \\
\hline
Collins et al.\cite{collins2014memory} & 55.1\% & 15.2M &  \xmark \\
Zhou et al. \cite{zhou2016less} & 54.4\% & 14.1M & \xmark \\
Han et al. \cite{han2015deep} &  57.2\% & 6.7M & \xmark  \\
Han et al. \cite{han2015deep} & 57.2\% & 6.7M & \xmark \\
 Srinivas et al. \cite{srinivas2017training} & 56.9\% & 5.9M & \xmark \\
Guo et al. \cite{guo2016dynamic} & 56.9\% & 3.4M & \xmark \\
X-AlexNet-1 & 55.2\% & 7.6M & \cmark \\
X-AlexNet-2 & 56.2\% & 9.7M & \cmark \\
\hline
AlexNet-Orig & 57.2\% & 61M & - \\
\hline
\end{tabular}
\captionof{table}{Comparison with other methods on ImageNet-2012 using AlexNet as the base model. We are able to achieve comparable accuracies using only 9.7M parameters.}
\label{tab:imagenet_fullcomp}
%\vspace{-0.6cm}
    \end{table}

\subsection{Stability of Models}
We give empirical evidence as well as a theoretical argument regarding the stability of our method. For the vanilla DNN training, the weights are randomly initialized, and randomized techniques like dropouts, augmentation are used. Hence there is some randomness present and is well accepted in DNN literature prior to our method. We repeat experiments on different datasets (Imagenet and CIFAR10) and architectures (VGG, DenseNet and MobileNet0.5) to empirically show that the accuracy of expander based models has variance similar to vanilla DNN training over multiple runs.

We repeated the experiments with independent sampling of random expanders on the VGG and DenseNet baselines on the CIFAR10 dataset. The results can be seen in Table~\ref{tab:cifar}. It is noted that the accuracy values changes only by less than $0.3$\% across runs and the standard deviation of expander method is also comparable to the vanilla DNN training.

We also repeated experiments of our main result, which is the comparison with grouped convolutions on ImageNet dataset. We rerun the experiment with MobileNet0.5 feature extractor twice with Groups and the expander method. As can be seen from Table~\ref{tab:imgnet}, the accuracy variations are comparable between the two models, and it is less than 1\%.

\begin{table}[!thb]
%\vspace{-2em}
\begin{minipage}{.5\linewidth}
\centering
\scalebox{0.9}{
\begin{tabular}{|l|c|c|c|c|}
\hline
{\bf Model} & {\bf Accuracy\%} & {\bf Max\%} & {\bf Min\%}\\
\hline
VGG & 93.96$\pm0.12$ & 94.17 & 93.67 \\
\hline
X-VGG-1 & 93.31$\pm0.18$ & 93.66 & 93.06 \\
X-VGG-2 & 92.91$\pm0.19$ & 93.26 & 92.69 \\
\hline
XDNetBC-40-24 & 94.41$\pm$0.19 & 94.63 & 94.18 \\
XDNetBC-40-36 & 94.98$\pm$0.14 & 95.21 & 94.84 \\
XDNetBC-40-48 & 95.49$\pm$0.15 & 95.65 & 95.28 \\
XDNetBC-40-60 & 95.75$\pm$0.07 & 95.81 & 95.68 \\

\hline
\end{tabular}}
%\vspace{0.5em}
\caption{The accuracies (mean $\pm$ stddev) of various models over 10 training runs on CIFAR-10 dataset.} \label{tab:cifar}
\end{minipage}%
\hfill
 \begin{minipage}{.45\linewidth}
 \centering
 \scalebox{0.9}{
\begin{tabular}{|c|c|c|}
\hline
{\bf MobileNet} & {\bf Mean} & {\bf Range}\\
{\bf Variant}  & {\bf Accuracy} & {\bf (Max-Min)}\\
\hline
Base & $63.39\%$ & $0.11\%$\\

\hline
G2 & $57.45~\%$ & $0.06~\%$\\
X2 & \bf{$58.22~\%$} & $ 0.14~\%$\\
\hline
G4 & $49.41~\%$ & $0.55~\%$\\
X4 & \bf{$54.00~\%$} & $0.53~\%$\\
\hline
G8 & $45.13~\%$ & $0.03~\%$\\
X8 & \bf{$49.23~\%$} & $0.60~\%$\\
\hline
G16 & $39.03~\%$ & $0.64~\%$\\
X16 & \bf{$44.63~\%$} & $0.18~\%$\\
\hline
\end{tabular}}
%\vspace{0.5em}\\
\caption{The mean accuracy and range of variation over  2 runs of MobileNet0.5 variants on ImageNet dataset.} \label{tab:imgnet}
\end{minipage}%
%\vspace{-3em}
\end{table}
A theoretical argument also concludes that choosing random graphs doesn't degrade stability. It is a well known result (See Theorem 4.4 in  \cite{pseudo}) in random graph theory, that graphs chosen randomly are well connected with overwhelmingly high probability (with only inverse exponentially small error, due to the Chernoff's Tail bounds) and satisfies the Expander properties. Hence the chance that for a specific run, the accuracy gets affected due to the selection of a particularly badly connected graph is insignificant.

\subsection{Training Wider and Deeper  networks}\label{sec:ultrawide}

\begin{figure*}[t]
\begin{tabular}{cc}

\includegraphics[scale=0.25]{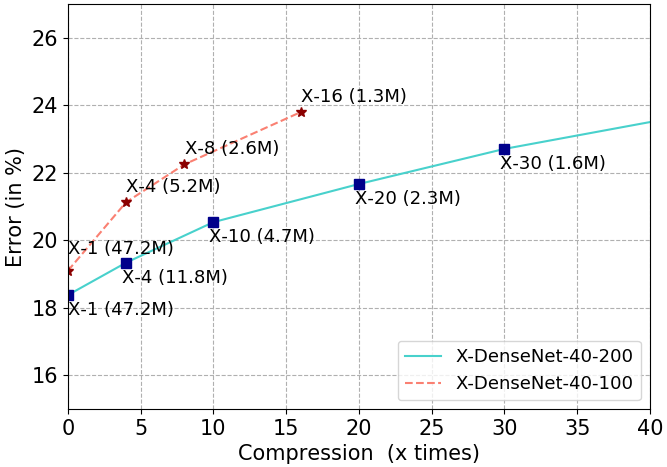}  & 
\includegraphics[scale=0.25] {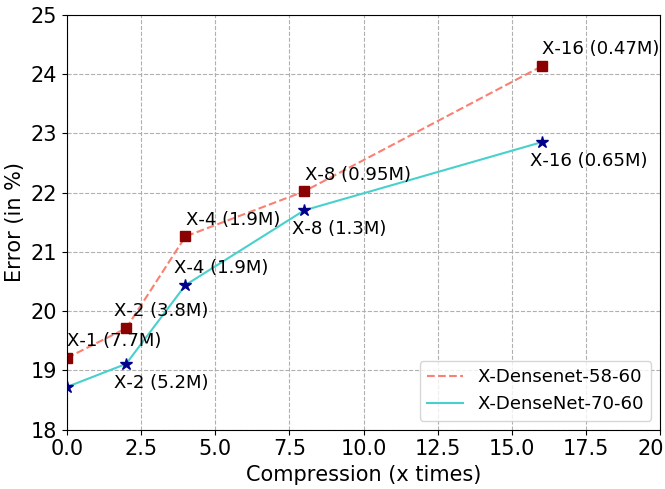}\\
(a) Effect of Width  & (b) Effect of Depth \\
\end{tabular}
\caption{We show the performance tradeoff obtained on training significantly wider and deeper networks on CIFAR-100 dataset. Every datapoint is X-$C$ specified along with the number of parameters, $C$ being the compression factor. We show that training wider or deeper networks along with more  compression using X-Nets achieve better accuracies with upto two-thirds of the total parameter and FLOPs on CIFAR-100 dataset. }
\label{fig:deepnet}
\end{figure*}

Since X-Nets involve constraining the weight matrices to sparse connectivity patterns before training, the fast algorithms can make it possible to utilize memory and runtime efficiently in training phase. This makes it possible to train significantly deeper and wider networks. Note the contrast with pruning techniques, where it is necessary to train the full, bulky model, inherently limiting the range of models that can be compressed.

Wide-DenseNets\footnote{https://github.com/liuzhuang13/DenseNet\#wide-densenet-for-better-timeaccuracy-and-memoryaccuracy-tradeoff} offered a better accuracy-memory-time trade-off. We increase the width and depth of these networks to train significantly wider and deeper networks. The aim is to study whether leveraging the effectiveness of X-Nets in this fashion can lead to better accuracies. 
  
We widen and deepen the DenseNet-BC-40-60 architecture, increasing the growth rate from 60 to 100 and 200 respectively and compare the effect of increasing width on these new models. Similarly, we increase the depth from 40 to 58 and 70 to obtain deeper networks. We benchmark these approaches using CIFAR-100 dataset and present the results in Figure \ref{fig:deepnet}. 

We have two interesting observations. First, the deeper X-DenseNet-BC-70-60 significantly outperforms X-DenseNet-BC-58-60 and wider X-DenseNet-40-200 outperforms X-DenseNet-BC-40-100 with fewer parameters for a wide range of $C$ values (Expander degree). 

The second interesting observation is the decreasing slope of the curves. This  indicates that expander graph modeling seems to be effective on wider and deeper X-Nets i.e X-DenseNet-BC models suffer lesser penalty with increasing depth and width compression. This enables X-Nets to work at high compression rates of 30x, compressing DenseNet-BC-40-200 model from 19.9B FLOPs to 0.6B FLOPs with only $4.3\%$ drop in accuracy. We hope this preliminary investigation holds significant value in alleviating the constraint of GPU memory and resources.

\section{Conclusion}
We proposed a new network layer architecture for deep networks using expander graphs that give strong theoretical guarantees on connectivity. The resulting architecture (X-Net) is shown to be highly efficient in terms of both computational requirements and model size. In addition to being compact and computationally efficient, the connectivity properties of the network allow us to achieve significant improvements over the state-of-the-art architectures in performance on a parameter or run-time budget. In short, we show that the use of principled approaches that sparsify a model while maintaining global information flows can help in developing efficient deep networks.

To the best of our knowledge, this is the first attempt at using theoretical results from graph theory in modeling connectivity to improve deep network architectures. We believe that the field of deep networks can gain significantly from other similar explorations.

\bibliographystyle{splncs}
\bibliography{egbib}
\ifarxiv
\appendix
\ifarxiv
\else
This appendix provides formal proofs for the theorems stated in Section 3 of the main paper. We also include tabulation of results of Figure 3 and 4, alongside additional results that could not be added in the paper due to space constraints.
\fi

\section{Explicit Expanders}
 Many constructions of expander graphs have been explored in the past (see \cite{expsurvey}). We will be using a construction that is comparatively easy to describe and implement. We will be considering Cayley graphs, which are obtained from the theory of finite Fields/Groups. The vertex set of such graphs is a group and the edges are defined by addition operation on the vertices. Also we will be describing expanders that are simple undirected graphs. Given an undirected graph $G=(V,E)$, we can obtain  a bipartite expander $G=(V,V',E)$, by making a copy of the vertices $V$ on the other side and adding edges according to $E$.

\subsection{Cayley Expander}
Let $V$ be a group under an operation $+$ and let $H \subset V$ be a set of generators of the group $V$. The Cayley graph defined by $V,H$ is the graph with vertex set $V$ and the edges $E =\{(x,x+h) : h \in H\}$. For our construction, we will consider the group $\{0,1\}^n$ under coordinate-wise XOR operations. For $H \subset \{0,1\}^n$, the Cayley graph defined by $H$ is $|H|$-regular. It is a well-known result in spectral graph theory that there are set of generators $H$, for which the Cayley graph defined by $\{0,1\}^n, H$ forms an expander with spectral gap $\gamma$.
\begin{theorem*}[Alon, Roichman, Proposition 4 \cite{cayley}]
For every $\epsilon$,  there exists explicit $H \subset \{0,1\}^n$ of size $\leq O(n^2/\epsilon^2)$ such that the Cayley graph defined by $\{0,1\}^n, H$ is an expander with spectral gap $\gamma = 1-\epsilon$.
\end{theorem*}

\section{Sensitivity in Expanders}
In this section, we will give details about some properties of expander graphs that are required for proving Theorems 1 and 2. Expander graphs are characterized by the spectral gap, which in turn implies the connectivity properties listed in Section 3.4.

The properties given in Section 3.4 follows from the fact that expanders have a spectral gap (for proofs see \cite{pseudo}). For the sensitivity proofs, we need the following lemmas. The first lemma states that having a spectral gap, implies that all set of vertices of size $\leq n/2$ expands.
\begin{lemma}[Theorem 4.6 \cite{pseudo}] 
\label{lab:spec}
If $G=(V,E)$ is an expander with spectral gap of $\gamma$ then for every subset $S\subset V$ of size $\leq |V|/2$, the size of the set of neighbors $|N(S)| \geq (1+\gamma) |S|$.
\end{lemma}
The next lemma is known as the expander mixing lemma, deals with the uniform connectivity properties of the expander.
\begin{lemma}[Lemma 4.15 \cite{pseudo}]
\label{lab:exp-mix-lem}
If $G=(V,E)$ is an $D$-regular expander with spectral gap of $\gamma$ then for every subset $S,T\subset V$,
$$\left| E(S,T) - D \cdot |S| \cdot |T| / n  \right| \leq (1-\gamma)\sqrt{|S| \cdot |T|}$$
where $E(S,T)$ is the set of edges from $S$ to $T$.
\end{lemma}

In this section, we prove Theorem 1 and Theorem 2 using the properties described above along with the connectivity properties defined in Section 3.1 of the paper.
\begin{theorem}[Sensitivity of X-Nets]\label{thm:conn}
Let $n$ be the number of input as well as output nodes in the network and $G_1,G_2,\cdots, G_t$ be $D$ regular bipartite expander graphs with $n$ nodes on both sides. Then  
every output neuron is sensitive to every input in a Deep X-Linear Network defined by $G_i$'s with depth $t = O( \log n)$.
\end{theorem}
\begin{proof}
For showing sensitivity, we show that for every pair of input and output $(u,v)$,  there is a path in the X-Net with the $i^{th}$ edge from the $i^{th}$ graph. 
We use the expansion property of the expander graphs. Let $N_1(u)$ be the set of neighbors of $u$ in $G_1$ and $N_i(u)$ be the set of neighbors of $N_{i-1}(u)$ in the graph $G_i$. Since each of the graphs are expanding $|N_{i}(u)| \geq (1+\gamma) \times |N_{i-1}(u)|$, using Lemma \ref{lab:spec}. Since $(1+\gamma) > 1$, we can obtain that for $i = O(\log n)$, $|N_i(u)| \geq n/2$. Similarly we can start from $v$ and define $N_1(v)$ as the set of neighbors of $v$ in $G_t$ and $N_i(v)$ be the set of neighbors of $N_{i-1}(v)$ in the graph $G_{n-i -1}$. Due to the expansion property, for $j = O(\log n)$, $|N_j(v)| \geq n/2$. Now we choose $t = i + j = O(\log n)$ so that the $i^{th}$ graph from $1$ is the same as $j+1^{th}$ graph from $t$. Then we will have that $N_i(u) \cap N_j(v) \neq \phi$. That is there is some vertex in the $i^{th}$ graph that is both connected to $u$ and $v$, which implies that there is a path from $u$ to $v$.
\end{proof}

\begin{theorem}[Mixing in Deep Expander Networks]
Let $n$ be the number of input as well as output nodes in the network and $G_1,G_2,\cdots, G_t$ be $D$ regular bipartite expander graphs with $n$ nodes on both sides. Let $S,T$ be subsets of input and output nodes in the X-Linear Network defined by the $G_i$'s. The number of paths between $S$ and $T$ is $\approx D|S||T|/n$
\end{theorem}
\begin{proof}
First we prove the theorem for the case when $G_1 = G_2 \cdots = G_t = G$. Note that for $t=1$, the theorem is same as the expander mixing lemma (Lemma \ref{lab:exp-mix-lem}). For $t>1$, consider the graph of $t$ length paths denoted by $G^t$. An edge in this graph denotes that there is a path of length $t$ in $G$. Observe that the adjacency matrix of $G^t$ is given by the $t$th power of the adjacency matrix of $G$ and hence the spectral gap of $G^t$, $\gamma_t = \gamma^t$. Now applying the expander mixing lemma on $G^t$ (Lemma \ref{lab:exp-mix-lem}), proves the theorem.

For the case when the graph $G_i$'s are different, let $\gamma_{\min}$ be the minimal spectral gap among the graphs. Since all the graphs are $D$-regular, the largest eigenvector is the all ones vector with eigenvalue $D$ for all the graphs. The eigenvector corresponding to second largest eigenvalue can be different for each graph, but they are orthogonal to the all ones vector. Hence the spectral gap of the $t$ length path graph $G$ is at least $\gamma_{\min}^t$. Finally we apply the expander mixing lemma on $G^t$ to prove the theorem.
\end{proof}

\section{Model Details}

We discuss the model structures in detail, along with tabulated  values of size and flops of various architectures presented as graphs in the main paper. 

\begin{table}[!tbh]
\centering
\resizebox{0.6\columnwidth}{!}{
\begin{tabular}{|l|c|c|c|}
\hline
{\bf AlexNet} & {\bf Filter Shape} & {\bf Filter Shape} & {\bf Filter Shape} \\
\hline
 &  & ({\bf X-AlexNet-1}) & ({\bf X-AlexNet-2})\\
 \hline
Conv2d & 64 x 3 x 11 x 11 & 64 x 3 x 11 x 11 & 64 x 3 x 11 x 11\\
Conv2d & 192 x 64 x 5 x 5 & 192 x 64 x 5 x 5 & 192 x 64 x 5 x 5\\
Conv2d & 384 x 192 x 3 x 3 & 384 x 192 x 3 x 3 & 384 x 192 x 3 x 3\\
Conv2d & 256 x 384 x 3 x 3 & 256 x 384 x 3 x 3 & 256 x 384 x 3 x 3\\
Conv2d & 256 x 256 x 3 x 3 & 256 x 256 x 3 x 3 & 256 x 256 x 3 x 3\\
\hline
Linear & 9216 x 4096 & 1024 x 4096 & 512 x 4096\\
Linear & 4096 x 4096 & 512 x 4096 & 512 x 4096\\
Linear & 4096 x 1000 & 1024 x 1000 & 1024 x 1000\\
\hline
\end{tabular}}
% \vspace{0.1cm}
\caption{Filter sizes for the AlexNet model. Notice the filter sizes of the linear layers of the original model has $|V|\times |U|$ parameters, whereas X-AlexNet models have $|V|\times D$ parameters. Note that $D << |U|$ as stated in Section 3.2. Hence, expander graphs model connections in linear layers (X-Linear) effectively.}
\label{tab:vgg}
\end{table}
\begin{table}[!tbh]
\centering
\resizebox{0.6\columnwidth}{!}{
\begin{tabular}{|l|c|c|c|}
\hline
{\bf VGG} & {\bf Filter Shape} & {\bf Filter Shape} & {\bf Filter Shape}\\
\hline
 &  & ({\bf X-VGG16-1}) & ({\bf X-VGG16-2})\\
 \hline
Conv2d & 64 x 3 x 3 x 3 & 64 x 3 x 3 x 3 & 64 x 3 x 3 x 3\\
Conv2d & 64 x 64 x 3 x 3 & 64 x 64 x 3 x 3 & 64 x 64 x 3 x 3\\
Conv2d & 128 x 64 x 3 x 3 & 128 x 64 x 3 x 3 & 128 x 64 x 3 x 3\\
Conv2d & 128 x 128 x 3 x 3 & 128 x 64 x 3 x 3 & 128 x 64 x 3 x 3\\
Conv2d & 256 x 128 x 3 x 3 & 256 x 32 x 3 x 3 & 256 x 16 x 3 x 3\\
Conv2d & 256 x 256 x 3 x 3 & 256 x 32 x 3 x 3 & 256 x 16 x 3 x 3\\
Conv2d & 256 x 256 x 3 x 3 & 256 x 32 x 3 x 3 & 256 x 16 x 3 x 3\\
Conv2d & 512 x 256 x 3 x 3 & 512 x 32 x 3 x 3 & 512 x 16 x 3 x 3\\
Conv2d & 512 x 512 x 3 x 3 & 512 x 32 x 3 x 3 & 512 x 16 x 3 x 3\\
Conv2d & 512 x 512 x 3 x 3 & 512 x 32 x 3 x 3 & 512 x 16 x 3 x 3\\
Conv2d & 512 x 512 x 3 x 3 & 512 x 32 x 3 x 3 & 512 x 16 x 3 x 3\\
Conv2d & 512 x 512 x 3 x 3 & 512 x 32 x 3 x 3 & 512 x 16 x 3 x 3\\
Conv2d & 512 x 512 x 3 x 3 & 512 x 32 x 3 x 3 & 512 x 16 x 3 x 3\\
\hline
Linear & 512 x 512 & 128 x 512 & 128 x 512\\
Linear & 512 x 10 & 512 x 10 & 512 x 10\\
\hline
\end{tabular}}
% \vspace{0.1cm}
\caption{Filter sizes for the VGG-16 model on CIFAR-10 dataset. The filter sizes given are $|V|\times |U| \times c \times c$ in original VGG network, $|V|\times D \times c \times c$ in our X-VGG16 models. Note that $D << |U|$ as stated in Section 3.2. Hence, expander graphs model connections in Convolutional layers (X-Conv) effectively.}
\label{tab:alexnet}
\end{table}

\begin{table}[!tbh]
\centering
\resizebox{0.6\columnwidth}{!}{
\begin{tabular}{|l|c|c|c|}
\hline
{\bf Model} & {\bf Accuracy} & {\bf \#Params } & {\bf \#FLOPs}\\
\hline
CIFAR10 &  & {\bf(in M)} & {\bf(in 100M)}\\
\hline
X-DenseNetBC-2-40-24 & {\bf 94.83\%} & {\bf 0.4M} & {\bf 1.44}\\
DenseNetBC-40-24 & 94.79\% & 0.7M & 2.88\\
X-DenseNetBC-2-40-36 & 94.98\% & 0.75M & 3.24\\
X-DenseNetBC-2-40-48 & {\bf 95.48\%} & {\bf 1.4M} & {\bf 5.75}\\
DenseNetBC-40-36 & 95.26\% & 1.5M & 6.47\\
X-DenseNetBC-2-40-60 & {\bf 95.71\%} & {\bf 2.15M} & {\bf 8.98}\\
DenseNetBC-40-48 & 95.64\% & 2.8M & 11.50\\
DenseNetBC-40-60 & 95.91\% & 4.3M & 17.96\\
\hline
CIFAR100 &  &  &\\
\hline
X-DenseNetBC-2-40-24 & 74.37\% & 0.4M & 1.44\\
DenseNetBC-40-24 & 76.05\% & 0.7M & 2.88\\
X-DenseNetBC-2-40-36 & {\bf 76.69\%} & {\bf 0.75M} & {\bf 3.24}\\
DenseNetBC-40-36 & 77.84\% & 1.5M & 6.47\\
X-DenseNetBC-2-40-60 & 78.53\% & 2.15M & 8.98\\
X-DenseNetBC-4-70-60 & {\bf 79.56\%} & {\bf 2.6M} & {\bf 10.26}\\
DenseNetBC-40-48 & 79.03\% & 2.8M & 11.50\\
DenseNetBC-40-60 & 79.87\% & 4.3M & 17.96\\
X-DenseNetBC-2-70-60 & 80.89\% & 5.18M & 20.52\\
DenseNetBC-70-60 & 81.28\% & 10.36M & 41.05\\
\hline
\end{tabular}}
% \vspace{0.1cm}
\caption{Results obtained on the state-of-the-art models on CIFAR-10 and CIFAR-100 datasets, ordered by FLOPs per model. X-Nets give significantly better accuracies with corresponding DenseNet models in the same limited computational budget and correspondingly significant parameter and FLOP reduction for models with similar accuracy.}
\label{tab:cifar}
\end{table}
\subsection{Filter structure of AlexNet and VGG}

In Tables \ref{tab:vgg} and \ref{tab:alexnet}, the detailed layer-wise filter structure is tabulated as stated in Section 5.3 of the paper. We compare the sizes of input channels between the filters, and show that with X-Conv and X-Linear layers, we can train models effectively even with upto 32x and 18x times smaller filters in input dimension in VGG16 and AlexNet models respectively. Hence, modeling connections as weighted adjacency matrix of an Expander graph is an effective method to model connections between neurons, producing highly efficient X-Nets.

\subsection{Results}
As stated in Section 5.2, Tables \ref{tab:cifar} and \ref{tab:imagenet} presented below give the detailed accuracy, parameters and FLOPs of models displayed in the Figure 3 in the paper. Details of other models are also provided in the same table, which could not be displayed in the paper due to lack of space.

Table \ref{tab:cifar} displays the performance of X-DenseNet models on the CIFAR-10 and CIFAR-100 datasets. If we compare models that have similar number of parameters, we achieve around 0.2\% and 0.6\% increase in accuracy over DenseNet-BC models on CIFAR-10 and CIFAR-100 datasets respectively. In the same manner, we can achieve upto using only two-thirds of the parameter and runtime cost cost respectively, keeping accuracy constant on CIFAR-10 and CIFAR-100 datasets as stated in the paper.

\begin{table}[!tbh]
\centering
\resizebox{0.6\columnwidth}{!}{
\begin{tabular}{|l|c|c|c|}
\hline
{\bf Model} & {\bf Accuracy} & {\bf \#Params } & {\bf \#FLOPs}\\
\hline
ResNet &  & {\bf(in M)} & {\bf(in 100M)}\\
\hline
X-ResNet-2-34 & 69.23\% & 11M & 35\\
X-ResNet-2-50 & {\bf 72.85\%} & {\bf 13M} & {\bf 40}\\
ResNet-34 & 71.66\% & 22M & 70\\
X-ResNet-2-101 & {\bf 74.87\%} & {\bf 22.5M} & {\bf 80}\\
ResNet-50 & 74.46\% & 26M & 80\\
ResNet-101 & 75.87\% & 45M & 160\\
\hline
DenseNetBC &  &  & \\
\hline
MobileNet \cite{howard2017mobilenets} & 70.6\% & 4.2M & 5.7 \footnote{These are reported as mult-add operations} \\
ShuffleNet \cite{zhang2017shufflenet} & 70.9\% & 5M & 5.3 \footnote{These are reported as mult-add operations} \\
X-DenseNetBC-2-121 & {\bf 70.5\%} & {\bf 4M} & 28\\
X-DenseNetBC-2-169 & 71.7\% & 7M & 33\\
X-DenseNetBC-2-201 & 72.5\% & 10M & 43\\
X-DenseNetBC-2-161 & {\bf 74.3\%} & 14.3M & {\bf 55}\\
DenseNetBC-121 & 73.3\% & 8M & 55\\
DenseNetBC-169 & 74.8\% & 14M & 65\\
DenseNetBC-201 & 75.6\% & 20M & 85\\
DenseNetBC-161 & 76.3\% & 28.5M & 110\\
\hline
\end{tabular}}
% \vspace{0.1cm}
\caption{Results obtained on the state-of-the-art models on ImageNet dataset, ordered by FLOPs. We also observe that X-DenseNetBC models outperform ResNet and X-ResNet models in both compression, parameters and FLOPs and achieve comparable accuracies with the highly efficient MobileNets and ShuffleNets in the same parameter budget, albeit with much higher FLOPs due to architectural constraints.}
\label{tab:imagenet}
\end{table}

Similarly, Table \ref{tab:imagenet} displays the performance of X-ResNet and X-DenseNet models on the ImageNet datasets. We can observe that we achieve around 3.2\% and 1\% increase in accuracy over ResNet and DenseNet-BC models in the same computational budget. Also, we can observe that we require approximately 15\% less FLOPs for achieving similar accuracies over DenseNet models and 15\% less parameters for achieving similar accuracies over the ResNet model respectively as stated in the paper. We also observe that X-DenseNetBC models outperform ResNet and X-ResNet models in both compression, parameters and FLOPs and achieve comparable accuracies with MobileNets \cite{howard2017mobilenets} and ShuffleNets \cite{zhang2017shufflenet} in the same parameter budget, albeit with much higher computational cost due to architectural constraints. 

\begin{table}[!tbh]
\centering
\resizebox{0.6\columnwidth}{!}{
\begin{tabular}{|l|c|c|c|}
\hline
{\bf Model} & {\bf Accuracy} & {\bf \#Params } & {\bf \#FLOPs}\\
\hline
Wider & & {\bf(in M)} & {\bf(in 100M)} \\
\hline
DenseNetBC-40-60 & 79.87\% & 4.3M & 17.96\\
X-DenseNetBC-2-40-60 & 78.53\% & 2.15M & 8.98\\
X-DenseNetBC-4-40-60 & 77.54\% & 1.08M & 4.49\\
X-DenseNetBC-8-40-60 & 75.29\% & 0.54M & 2.24\\
X-DenseNetBC-16-40-60 & 74.44\% & 0.27M & 1.12\\
\hline
DenseNetBC-40-100 & 80.9\% & 11.85M & 49.85\\
X-DenseNetBC-4-40-100 & 78.87\% & 2.9M & 12.46\\
X-DenseNetBC-8-40-100 & 77.75\% & 1.48M & 6.23\\
X-DenseNetBC-16-40-100 & 76.2\% & 0.74M & 3.12\\
\hline
DenseNetBC-40-200 & 81.62\% & 47.19M & 199.28\\
X-DenseNetBC-4-40-200 & 80.66\% & 11.79M & 49.82\\
X-DenseNetBC-10-40-200 & 79.46\% & 4.71M & 19.93\\
X-DenseNetBC-20-40-200 & 78.33\% & 2.3M & 9.96\\
X-DenseNetBC-30-40-200 & 77.29\% & 1.6M & 6.64\\
X-DenseNetBC-50-40-200 & 75.7\% & 0.9M & 3.99\\
X-DenseNetBC-80-40-200 & 73.26\% & 0.5M & 2.49\\
\hline
Deeper &  &  &\\
\hline
DenseNetBC-40-60 & 79.87\% & 4.3M & 17.96\\
X-DenseNetBC-2-40-60 & 78.53\% & 2.15M & 8.98\\
X-DenseNetBC-8-40-60 & 77.54\% & 0.54M & 2.24\\
X-DenseNetBC-16-40-60 & 75.29\% & 0.27M & 1.12\\
\hline
DenseNetBC-58-60 & 80.79\% & 7.66M & 30.96\\
X-DenseNetBC-2-58-60 & 80.29\% & 3.83M & 15.48\\
X-DenseNetBC-4-58-60 & 78.74\% & 1.9M & 7.74\\
X-DenseNetBC-8-58-60 & 77.98\% & 0.95M & 3.87\\
X-DenseNetBC-16-58-60 & 75.87\% & 0.47M & 1.93\\
\hline
DenseNetBC-70-60 & 81.28\% & 10.36M & 41.05\\
X-DenseNetBC-2-70-60 & 80.89\% & 5.18M & 20.52\\
X-DenseNetBC-4-70-60 & 79.56\% & 2.6M & 10.26\\
X-DenseNetBC-8-70-60 & 77.48\% & 1.3M & 5.13\\
X-DenseNetBC-16-70-60 & 77.23\% & 0.65M & 2.57\\
\hline
\end{tabular}}
% \vspace{0.1cm}
\caption{We display accuracies, parameters and FLOPs of all the wider and deeper networks on CIFAR-100 listed in increasing compression order. This proves that efficiently designing layers like X-Conv and X-Linear allows us to train wider and deeper networks frugally.}
\label{tab:ultranets}
\end{table}

Table \ref{tab:ultranets} displays accuracies, parameters and FLOPs of all the wider and deeper networks trained on CIFAR-100 dataset as discussed in Section 5.4. They are listed in increasing compression order from DensenetBC (1x) to the highest compressed X-DenseNet. The figure indicates that Expander Graphs modeling can scale up to high compression ratios without drastic drops in accuracies, enabling us to train deeper and wider networks retaining similar FLOPs and parameters effectively. We believe this modeling can open up a interesting exploration of training significantly deeper and wider range of models, unlike the current compression techniques as X-Nets are highly compressed networks since definition. 

\subsection{Experimental Details}
\label{sec:expdetails}
To ensure better reproducibility, we used the same hyper-parameters, models, training schedules and dataset structures from the official \href{https://github.com/pytorch/examples/blob/master/imagenet/main.py}{PyTorch repository} for ResNet and DenseNet-BC ImageNet experiments. Similarly, we followed the \href{https://github.com/marvis/pytorch-mobilenet}{pytorch-mobilenet repository} including all hyperparameters for all the mobilenet experiments. Likewise, we followed the \href{https://github.com/andreasveit/densenet-pytorch/}{densenet-pytorch repository} by Andreas Veit for all DenseNet-BC experiments on CIFAR-10/100 datasets, and the \href{https://github.com/chengyangfu/pytorch-vgg-cifar10}{pytorch-vgg-cifar10} repository by Cheng-Yang Fu for VGG16 experiments on the CIFAR-10 dataset. We deleted a linear layer and remove dropouts from the VGG16 model. There were two differences between our code and the PyTorch ImageNet repository. Our ImageNet training code used a compressed version of Imagenet with images resized to 256x256, we used a batchsize of 128 for all experiments, hence typically results in slightly lower accuracies. Our AlexNet model used BatchNorm layers to stabilize training, with a batchsize of 384.

We trained all the models on a setup consisting of 10 Intel Xeon E5-2640 cores and 2 GeForce GTX 1080 Ti GPUs. All networks are trained from scratch. We modeled connections as  random expanders for all experiments. The X-Conv layers did not have a bias. No dropouts were used. Additional details regarding all models, accuracies, and numbers of parameters and FLOPs is tabulated in the supplementary material for reference due to lack of space.  

\par
All models in the plot were trained from the ImageNet code available on Pytorch Official repository including the original DenseNet and ResNet models. Note that this training code is common for all models in the repository and not fine-tuned to any specific model like ResNet or DenseNet. Hence the accuracies we report is slightly lower that those reported using model-specific training code. However, note that we have used the same code for training both the original models and the expander versions of these models. Hence the comparison is a fair one. Same is true with the  AlexNet model. Note that we use the original AlexNet architecture, and not CaffeNet. 
In contrast, training with fine-tuned code is expected to give a 1-2\% improvement over MobileNets in Table \ref{tab:imagenet}. Training schedules suited to X-Nets, hyper-parameter tuning, Activation layers tailored to X-Nets could further improve the accuracies. Overall, we believe that investigating training methods for X-Nets has a lot of potential for improving their performance.

\else
\fi

\end{document}